\documentclass[letterpaper]{article} 
\usepackage{aaai25}  
\usepackage{times}  
\usepackage{helvet}  
\usepackage{courier}  
\usepackage[hyphens]{url}  
\usepackage{graphicx} 
\usepackage{natbib}  
\usepackage{caption} 
\usepackage[linesnumbered,ruled,vlined]{algorithm2e}

\usepackage{algorithmic}

\usepackage{mathrsfs}
\usepackage{amsmath}
\usepackage{amsthm}
\usepackage{amssymb}

\newcommand{\enc}[1]{[\![#1]\!]}
\newcommand{\comment}[1]{#1}

\newtheorem{thm}{Theorem}
\newtheorem{cor}{Corollary}
\newtheorem{lem}{Lemma}
\newtheorem{prop}{Proposition}
\theoremstyle{definition}
\newtheorem{defn}{Definition}
\theoremstyle{remark}

\theoremstyle{definition}
\newtheorem{exm}{Example}

\title{A Theory of Formalisms for Representing Knowledge\comment{\\(Extended Version)}
 }
\author{
    Heng Zhang\textsuperscript{\rm 1}\thanks{Corresponding author.},
    Guifei Jiang\textsuperscript{\rm 2},
    Donghui Quan\textsuperscript{\rm 1}
}
\affiliations{
    \textsuperscript{\rm 1}Zhejiang Lab, Hangzhou, Zhejiang 311121, China\\
    \textsuperscript{\rm 2}College of Software, Nankai University, Tianjin 300350, China\\
    h.zhang@zhejianglab.org, g.jiang@nankai.edu.cn, donghui.quan@zhejianglab.org

}

\usepackage{bibentry}

\begin{document}

\maketitle

\begin{abstract}
There has been a longstanding dispute over which formalism is the best for representing knowledge in AI. The well-known “declarative vs. procedural controversy” is concerned with the choice of utilizing declarations or procedures as the primary mode of knowledge representation. The ongoing debate between symbolic AI and connectionist AI also revolves around the question of whether knowledge should be represented implicitly (e.g., as parametric knowledge in deep learning and large language models) or explicitly (e.g., as logical theories in traditional knowledge representation and reasoning). To address these issues, we propose a general framework to capture various knowledge representation formalisms in which we are interested. Within the framework, we find a family of universal knowledge representation formalisms, and prove that all universal formalisms are recursively isomorphic. Moreover, we show that all pairwise intertranslatable formalisms that admit the padding property are also recursively isomorphic. These imply that, up to an offline compilation, all universal (or natural and equally expressive) representation formalisms are in fact the same, which thus provides a partial answer to the aforementioned dispute.
\end{abstract}

%
 \begin{links}
     \link{Extended version}{https://arxiv.org/abs/2412.11855}
 \end{links}

\section{Introduction}

Knowledge is extensively acknowledged as a cornerstone of intelligence \cite{McCarthy1981,DelgrandeG0TW24}, playing a crucial role in intelligent systems. How to effectively represent, acquire, utilize and evolve knowledge is undoubtedly one of the most critical parts of realizing artificial general intelligence (AGI). Knowledge representation serves as the foundation and starting point for all these tasks. In traditional knowledge representation and reasoning (KR), representations of knowledge are often regarded as ``explicit, symbolic, declarative representations of information"~\cite{DelgrandeG0TW24}. However, in this work, we will consider more general forms of knowledge representation.

Over the past seven decades, researchers have devoted substantial efforts to developing various knowledge representation formalisms. An incomplete list includes: monotonic logical systems such as Prolog~\cite{EmdenK76} and description logics~\cite{BHLS17}; non-monotonic logics such as circumscription~\cite{McCarthy80} and default logic~\cite{Reiter80}; graph-based representations such as Bayesian networks~\cite{Pearl85} and semantic network~\cite{Sowa91}, and parametrized models such as recurrent neural networks~\cite{RHW86} and transformers~\cite{VaswaniSPUJGKP17}.
It is noteworthy that machine learning and knowledge representation are inherently intertwined; all learning algorithms are actually based on some formalisms for representing knowledge.

The quest for the best formalism of knowledge representation has sparked a longstanding dispute
. A prime example is the ``declarative vs. procedural" controversy, which centers on choosing between declarative statements and procedures as the primary means of representing knowledge. Similarly, the ongoing debate between symbolic AI and connectionist AI revolves around the question of whether knowledge should be represented explicitly or implicitly. In general, knowledge representation formalisms in machine learning, such as convolutional neural networks in deep learning and transformers in large language models, are  implicit, while all logical formalisms in traditional KR are explicit.

In this work, we will undertake a systematic exploration of disputed issues, particularly the search for the optimal knowledge representation formalism. A general framework is needed to systematically evaluate varied formalisms. While extensive philosophical deliberations on fundamentals of knowledge representation exist, including the physical symbol system hypothesis~\cite{NewellS76} and the knowledge representation hypothesis~\cite{Smith82}, they primarily remain within the realm of theoretical consensus-building. Departing from this path, our goal is to obtain rigorous conclusions from a broad and inclusive framework through meticulous mathematical demonstrations, thereby deepening the understanding of the nature of representation.

The main contributions of this work are threefold. Firstly, we propose a general framework to capture all the knowledge representation formalisms in which we are interested, and  propose a novel definitions for universal formalisms. Secondly, we find a family of universal formalisms; based on them, we then prove that all possible universal formalisms are recursively isomorphic. Thirdly, we show that, under a natural condition, all subrecursive formalisms that can be translated into each other are recursively isomorphic. These results show us that, if an offline compilation is allowed, almost all the natural representation formalisms with the same expressive power can be regarded as the same, which thus provides us a partial answer to the aforementioned dispute.

\section{Conventions and Notations}
Suppose $A, B$ and $C$ are sets. By $p:A\rightharpoonup B$ we denote that $p$ is a {\em partial function} (or {\em mapping}) from $A$ to $B$. We call $p$ a partial function {\em on} $A$ if $A=B$. Let $\textit{dom}(p)$ and $\textit{ran}(p)$ denote the {\em domain} and {\em range} of $p$, respectively. In addition, we call $p$ a {\em function} from $A$ to $B$, written $p:A\rightarrow B$, if $p$ is {\em total}. Given $X\subseteq A$, by $p|_X$ we denote the {\em restriction} of $p$ to $X$. A partial function $q$ is called  an {\em extension} of $p$, or equivalently, $q$ {\em extends} $p$, if $p$ is a restriction of $q$. Given functions $p:A\rightarrow B$ and $q:B\rightarrow C$, by $p\circ q$ we denote the {\em composition} of $p$ and $q$.
All other notions of functions such as injectiveness, surjectiveness and bijectiveness are standard. 

The reader is assumed familiar with logic. A {\em signature} is a set that consists of {\em predicate} and {\em function symbols}, each associated with a nonnegative integer, called its {\em arity}. {\em Constants} are nullary function symbols. Let $\sigma$ be a signature. By {\em $\sigma$-atoms} and {\em $\sigma$-sentences} (or {\em sentences of $\sigma$}) we denote atoms and sentences, respectively, built from $\sigma$ and standard logical connectives and quantifiers as usual. Let FO and SO denote the classes of first-order and second-order sentences, respectively. A {\em structure of $\sigma$} (or simply, {\em $\sigma$-structure}) $\mathcal{A}$ is armed with a nonempty {\em domain} $A$, maps each predicate symbol $P\in\sigma$ to a relation $P^{\mathcal{A}}$ on $A$, and maps each function symbol $f\in\sigma$ to a function $f^{\mathcal{A}}$ on $A$, both are of the same arity. A {\em UNA-structure} of $\sigma$ is a $\sigma$-structure $\mathcal{A}$ that satisfies the unique name assumption, i.e., $c^{\mathcal{A}}\ne d^{\mathcal{A}}$ for every pair of distinct constants $c,d$ in $\sigma$. Given $\upsilon\subseteq\sigma$, let $\mathcal{A}|_{\upsilon}$ denote the {\em restriction} of $\mathcal{A}$ to $\upsilon$. We call $\mathcal{A}$ a {\em $\sigma$-expansion} of $\mathcal{B}$ if $\mathcal{B}=\mathcal{A}|_{\upsilon}$ for some $\upsilon\subseteq\sigma$. Let $\phi$ be a $\sigma$-sentence. We write $\mathcal{A}\models\phi$ if $\mathcal{A}$ is a {\em model} of $\phi$. Given a class $\mathbb{C}$ of $\sigma$-structures, we write $\mathbb{C}\models\phi$ if $\mathcal{A}\models\phi$ for all $\mathcal{A}\in\mathbb{C}$. Given a set $\Sigma$ of sentences and a sentence $\psi$, we write $\Sigma\vDash\phi$ and $\psi\vDash\phi$ if $\phi$ is a {\em logical implication} of $\Sigma$ and $\psi$, respectively. 

Every {\em Turing machine} $M$ is armed with a two-way infinite tape, a reading head, a finite set $Q$ of states and a fixed symbol set $\{0,1,B\}$. There is exactly one {\em starting state} and at least one {\em halting state} in $Q$. Every halting state is either an {\bf yes} state or a {\bf no} state, but cannot be both. Both the input and output are strings in $\{0,1\}^\ast$, stored in the tape, starting from the position of the reading head and ending with $B$.  Let $L\subseteq\{0,1\}^\ast$. We say $M$ {\em accepts} $L$ if, for every $\pi\in\{0,1\}^\ast$, $M$ {\em accepts} $\pi$ (i.e., $M$ on input $\pi$ halts at an {\bf yes} state) if $\pi\in L$, and never halts otherwise; and $M$ {\em decides} $L$ if, for every $\pi\in\{0,1\}^\ast$, $M$ accepts $\pi$ if $\pi\in L$, and rejects $\pi$ otherwise.  
We say $L$ is {\em recursively enumerable} (respectively, {\em recursive}) if it is accepted (respectively, decided) by some Turing machine. Moreover, Turing machines can also be used to compute functions. 
We say a partial function $p$ from $\{0,1\}^\ast$ to $\{0,1\}^\ast$ is {\em computed} by $M$ if, given $\pi\in\{0,1\}^\ast$ as input, $M$ halts with the output $\omega$ iff $p$ is defined on $\pi$ and $p(\pi)=\omega$. We say $p$ is {\em partial recursive} if it is computed by some Turing machine, and $p$ is {\em recursive} if it is partial recursive and $\textit{dom}(p)$ is recursive.

To simplify the presentation, we will fix $\enc{\cdot}$ as an injective mapping that maps every finite object to a string in $\{0,1\}^\ast$. For example, given a Turing machine $M$, by $\enc{M}$ we denote the encoding of $M$ in $\{0,1\}^\ast$. Moreover, we require that both $\enc{\cdot}$ and its inverse can be effectively obtained.

\section{Framework}
The major goal of this work is to carry out a careful comparison between different formalisms for representing knowledge. To this end, we have to propose a general framework that captures all the formalisms in which we are interested.

To build the desired framework, one immediate thought might be to define a family of abstract logical formalisms, similar to abstract logical systems proposed for establishing Lindstr\"{o}m's theorem~\cite{Lindstrom69}. But we do not pursue this approach in this work. The main reasons are as follows. Firstly, the framework established in this way is not general enough. It is important to note that logic is not the only method for representing knowledge. Secondly, from a user's perspective, the internal logical semantics of a representation formalism are actually not important. Users are primarily concerned with the outputs generated from given inputs. This aligns with a behaviorist perspective.

Regarding the primary computational task, we will focus on knowledge reasoning. While there are certainly other important tasks, such as knowledge acquisition (learning) and maintenance, knowledge reasoning typically runs online, with its efficiency directly determining the performance of the underlying system. In contrast, knowledge acquisition and maintenance can in general be performed offline.


We aim to go beyond the traditional reasoning problem to tackle a more general computational problem, known as {\em query answering} (QA). The problem of QA has been extensively studied in databases, see, e.g.,~\cite{FaginKMP05} and was later introduced into KR to implement data-intensive knowledge reasoning, see, e.g.,~\cite{CalvaneseGLLR07}.

The problem of QA in KR is defined as follows:
\begin{quote}
Given a database $D$, a knowledge base $K$ and a query $\phi$, determine whether $\phi$ is inferable from $D$ and $K$. 
\end{quote}
Intuitively, $D$ stores the {\em observed facts}, $\phi$ describes the {\em question} that the user want to ask, and $K$ represents the {\em knowledge} needed to answer the questions. It should be noted that if $D$ is empty, QA degenerates into the {\em traditional reasoning problem}; if $\phi$ is restricted to a proposition symbol, QA simplifies to both the {\em query evaluation problem} in databases and the {\em classification problem} in machine learning.
 
Following the behaviorist perspective, a notion of {\em abstract knowledge base} can then be defined as {\em the class of database-query pairs $(D,\phi)$ such that $\phi$ is inferable from $D$ and the underlying knowledge base}. To define this formally, we need to establish what constitutes valid databases and queries.

{\medskip\noindent\bf Databases and Queries.  }
We assume $\Delta$ to be a countably infinite set, consisting of all the constants used in databases and queries. Following the tradition in databases, both the closed-world assumption (CWA) and the open-world assumption (OWA) can be made~\cite{AbiteboulHV95}. Therefore, each predicate symbol is either an OWA-predicate symbol or a CWA-predicate symbol, but not both.

%

A {\em database signature} is a signature $\sigma\supseteq\Delta$ that involves no function symbols of arities greater than $0$. A {\em query signature} is a signature $\upsilon\supseteq\Delta$, containing no CWA-predicate symbol. Given any database signature $\sigma$,  let $\textit{Fact}(\sigma)$ denote the set of all $\sigma$-atoms that involve no variables and equality.

\begin{defn}
Let $\sigma$ be a database signature. A {\em $\sigma$-database} is a partial function $D:\textit{Fact}(\sigma)\rightharpoonup\{1,0,-1\}$ such that
\begin{enumerate}
\item (finiteness of observation) there are only a finite number of atoms $\alpha\in\textit{dom}(D)$ such that $D(\alpha)\ge 0$;
\item (completeness of CWA-predicates) $D$ is defined on every atom that involves a CWA-predicate symbol in $\sigma$. 
\end{enumerate} 
\end{defn}

Intuitively, in the above definition, by $D(\alpha)=1$ (respectively, $D(\alpha)=0$) we mean that $\alpha$ was observed to be true (respectively, false), and by $D(\alpha)=-1$ we mean that $\alpha$ has not been observed yet, but its truth is already determined by the current observation and a fixed set of rules under CWA.  

Every {\em observed fact} of $D$ is an atom $\alpha\in\textit{dom}(D)$ such that $D(\alpha)\ge 0$. Let $\textit{DC}(D)$ denote the set of constants each of which appears in at least one observed fact of $D$. Moreover, $D$ is said to be {\em positive} if there is no atom $\alpha\in\textit{dom}(D)$ such that $D(\alpha)=0$, i.e., no negative fact is allowed in $D$.  

We are interested in the following classes of databases: 

\begin{enumerate}
\item $\mathscr{D}^{\sigma}_{\mathrm{All}}$: the class of arbitrary $\sigma$-databases;
\item $\mathscr{D}^{\sigma}_{\mathrm{Pos}}$: the class of positive $\sigma$-databases.
\end{enumerate}

Let $\mathcal{A}$ be a structure of some signature $\upsilon\supseteq\sigma$. We say that $\mathcal{A}$ is a {\em model} of $D$, written $\mathcal{A}\models D$, if we have both
\begin{enumerate}
\item $\mathcal{A}\models\alpha$ for all $\alpha\in\textit{dom}(D)$ with $D(\alpha)=1$, and
\item $\mathcal{A}\not\models\alpha$ for all $\alpha\in\textit{dom}(D)$ with $D(\alpha)=0$.
\end{enumerate}

Next, we define what constitutes a query language:

\begin{defn}\label{defn:query_lang}
Given a query signature $\sigma$, a {\em query language} of $\sigma$ is a recursive class $\mathscr{Q}$ of FO-sentences of $\sigma$ such that 
\begin{enumerate}
\item $\mathscr{Q}$ is closed under conjunctions, that is, if $\phi,\psi\in\mathscr{Q}$, then $\phi\wedge\psi\in\mathscr{Q}$;
\item $\mathscr{Q}$ is closed under constant renaming, that is, if $\tau:\Delta\rightarrow\Delta$ is injective and $\phi\in\mathscr{Q}$, then $\tau(\phi)\in\mathscr{Q}$;
\item $\mathscr{Q}$ contains at least one non-tautological sentence. 
\end{enumerate}
The notation $\tau(\phi)$ above denotes the sentence obtained from $\phi$ by replacing every occurrence of each $c\in\Delta$ with $\tau(c)$.
\end{defn}

\begin{exm}
Both Boolean conjunctive queries (CQs) and unions of conjunctive queries (UCQs, i.e., existential positive FO-sentences) are query languages according the above definition, see, e.g.,~\cite{AbiteboulHV95}. 
\end{exm}

We believe that employing first-order fragments as query languages is a reasonable assumption for the following reasons. According to Lindström's second theorem, first-order logic is the most expressive semi-decidable logic that admits the Löwenheim-Skolem property~\cite{Lindstrom69}. It is also worth to mention that most of the results presented in this paper can be generalized to other semi-decidable logics.


{\medskip\noindent\bf Knowledge Bases. } To simplify the presentation, in the rest of this paper, we {\em fix $\sigma_D$ as a database signature, $\sigma_Q$ a query signature, $\mathscr{D}\in\{\mathscr{D}_{\mathrm{All}}^{\sigma_D},\mathscr{D}_{\mathrm{Pos}}^{\sigma_D}\}$, and $\mathscr{Q}$ a query language of $\sigma_Q$.} Now, let us present a definition for abstract knowledge bases, following the spirit of abstract OMQA-ontology in~\cite{ZhangZY16,ZJ22}.

\begin{defn}\label{defn:kb}
 A {\em knowledge base (KB)} over $(\mathscr{D},\mathscr{Q})$ is a subclass ${K}$ of $\mathscr{D}\times\mathscr{Q}$ satisfying all the following properties:
\begin{enumerate}
\item (Correctness of tautological queries) If $\phi\in\mathscr{Q}$ is a tautology and $D\in\mathscr{D}$, then $(D,\phi)\in K$;
\item (Closure under query implications) If 
$(D,\phi)\in{K}$ and $\psi\in \mathscr{Q}$ and $\phi\vDash\psi$, then $(D,\psi)\in{K}$;
\item (Closure under query conjunctions) If $(D,\phi)\in{K}$ and 
$(D,\psi)\in{K}$, then $(D,\phi\wedge\psi)\in{K}$;
\item (Closure under database extensions) If
$(D,\phi)\in{K}$ and $D_0\in\mathscr{D}$ extends $D$, then $(D_0,\phi)\in{K}$;
\item (Closure under constant renaming) If $(D,\phi)\in{K}$ and $\tau:\Delta\rightarrow\Delta$ is injective, then $(\tau(D),\tau(\phi))\in{K}$.
\end{enumerate}
The notation $\tau(\cdot)$ above is the same as that in Definition~\ref{defn:query_lang}, and it is naturally generalized to databases.  
\end{defn}

In the above, almost all properties are natural and easy to understand. We only give explanations for Properties 4 and 5. Intuitively, Property 4 states that reasoning about open-world information should be monotone, i.e., adding new observed OWA-facts will not change previous answers. Note that, by the definition of database, all CWA-predicates are information complete so that no CWA-fact can be added to a database (to build an extension), which means that Property 4 cannot be applied to any CWA-predicate.

Property 5 rests on the assumption that knowledge should encapsulate general properties applicable to all objects in the underlying domain, rather than including propositions about specific objects. Consequently, the names of objects should not influence the results of QA. 

In Definition~\ref{defn:kb}, only Boolean queries are used, but this is not limiting. By allowing constants in queries, we enable an efficient conversion from arbitrary QA to Boolean QA.


In machine learning, bounded-error algorithms are commonly used. Unfortunately, finding a meaningful method to evaluate error rates for representation-depend tasks like reasoning is extremely difficult, if not impossible~\cite{Lynch74}. This is why our framework does not account for this aspect.

Moreover, an important question arises as to whether the above properties (1-5) indeed capture the class of knowledge bases represented in any formalism in which we are interested. 
First consider the necessity. In a straightforward way, one can verify it case by case. The following is an example.

\begin{exm}
Suppose $\sigma_D$ contains no CWA-predicate symbols. Let $\Sigma$ be a set of sentences (in a monotone logic such as FO or SO) of a signature $\sigma\supseteq(\sigma_D\cup\sigma_Q)\setminus\Delta$. Let
$$K_\Sigma:=\{(D,\phi)\in\mathscr{D}\times\mathscr{Q}:D\cup\Sigma\vDash\phi\}.$$
It is easy to verify that $K_\Sigma$ is a KB over $(\mathscr{D},\mathscr{Q})$. 
\end{exm}

However, as aforementioned, there are a very large number of knowledge representation formalisms to be verified. To avoid this, we address the question from a semantical perspective. Let $D$ be a database recording the current observation in a certain domain. Based on the observation $D$, the knowledge base will produce a certain belief. The latter can be denoted by a class of worlds (structures) in which the belief holds.
We thus have the following definition:
 
\begin{defn}\label{defn:sm}
Let $\sigma\supseteq\sigma_D$ be a signature. A {\em belief mapping} of $(\sigma_D,\sigma)$ is a function $\mathbb{M}$ that maps every $\sigma_D$-database to a class of $\sigma$-structures such that
\begin{enumerate}
\item if $\mathcal{A}\in\mathbb{M}(D)$ then $\mathcal{A}\models D$;
\item if $\tau:\Delta\rightarrow\Delta$ is injective, and $\phi$ an FO-sentence of $\sigma$, then $\mathbb{M}(D)\models\phi$ iff $\mathbb{M}(\tau(D))\models\tau(\phi)$;
 \item if $D_0$ is an extension of $D$ and $\phi$ an FO-sentence of $\sigma$ such that $\mathbb{M}(D)\models\phi$, then $\mathbb{M}(D_0)\models\phi$
\end{enumerate}
for all $\sigma_D$-databases $D$ and $D_0$.
\end{defn}

In the above, the first condition states that the belief produced from the observation must be consistent with the observation; the second asserts that, up to a constant renaming, from the same observation, $\mathbb{M}$ produces the same belief; and the third denotes that adding new observed OWA-facts will not change answers obtained by query answering with $\mathbb{M}$.  

Next, we show how circumscription~\cite{McCarthy80} defines a belief mapping. Some notations are needed. 

The language of circumscription is the same as first-order logic, armed with the minimal model semantics. Let $\upsilon\supseteq\sigma_D$ be a signature. Let $\upsilon_c$ be the set of all CWA-predicate symbols in $\sigma_D$, and $\upsilon_o:=\upsilon\setminus\upsilon_c$. Let $\mathcal{A}$ and $\mathcal{B}$ be $\upsilon$-structures. We write $\mathcal{A}\subseteq_{\upsilon_c}\mathcal{B}$ if $\mathcal{A}$ and $\mathcal{B}$ share the same domain, and
\begin{enumerate}
\item for all $P\in\upsilon_c$, we have $P^{\mathcal{A}}\subseteq P^{\mathcal{B}}$, and
\item $\mathcal{A}|_{\upsilon_o}=\mathcal{B}|_{\upsilon_o}$, i.e., OWA-parts of $\mathcal{A}$ and $\mathcal{B}$ are the same.
\end{enumerate}

Furthermore, let $\Sigma$ be a set of FO-sentences of $\upsilon$, and let $D$ be a $\sigma_D$-database. We use
$\textit{Mod}^u_m(D,\Sigma,\upsilon_c)$ to denote the class of all UNA-structures of $\upsilon$ that are {\em $\subseteq_{\upsilon_c}$-minimal models} (i.e., minimal under the order $\subseteq_{\upsilon_c}$) of both $D$ and $\Sigma$.

\begin{exm}
Let $\mathbb{M}$ denote the mapping that maps each $\sigma_D$-database $D$ to $\textit{Mod}^u_m(D,\Sigma,\upsilon_c)$. It is easy to see that $\mathbb{M}$ is a belief mapping  as it satisfies Conditions 1-3 of Definition~\ref{defn:sm}.  
\end{exm}

With a belief mapping $\mathbb{M}$, the KB can then be defined: 
$$\textit{kb}(\mathbb{M},\mathscr{D},\mathscr{Q}):=\{(D,\phi)\in\mathscr{D}\times\mathscr{Q}:\mathbb{M}(D)\models\phi\}.$$

The following proposition tells us that, despite its excessive inclusiveness, every belief mapping defines a knowledge base satisfying all the properties of Definition~\ref{defn:kb}. 

\begin{prop}\label{prop:bm2kb}
Let $\sigma$ be a signature such that $\sigma_D\cup\sigma_Q\subseteq\sigma$, and $\mathbb{M}$ be a belief mapping of $(\sigma_D,\sigma)$. Then $\textit{kb}(\mathbb{M},\mathscr{D},\mathscr{Q})$ is a KB over $(\mathscr{D},\mathscr{Q})$. 
\end{prop}

Now, let us show the sufficiency of Properties 1-5 of Definition~\ref{defn:kb} to capture the notion of knowledge bases. It suffices to find a logical representation for each KB in Definition~\ref{defn:kb}. By logical representations, we use theories of McCarthy's circumscription under the unique name assumption. 

\begin{prop}\label{prop:kb_logic_impl}
Let $K$ be a KB over $(\mathscr{D},\mathscr{Q})$, and $\sigma$ be the set consisting of all CWA-predicate symbols in $\sigma_D$. Then there are a set $\Sigma$ of FO-sentences such that, for all $D\in\mathscr{D}$ and $\phi\in\mathscr{Q}$, $\textit{Mod}^u_m(D,\Sigma,\sigma)\models\phi$ iff $(D,\phi)\in K$. 
\end{prop}

\begin{proof}[Sketched Proof]
The main idea involves constructing a rule for each pair $(D,\phi)\in K$ such that its body describes $D$ and its head records $\phi$. If the facts in $D$ have been observed, the rule will be triggered to support QA on $\phi$. Let $\Sigma$ be the set of all such rules. We can prove that $\Sigma$ is the desired set.
\end{proof}

{\smallskip\noindent\bf Formalisms. } Based on the definition of abstract knowledge bases, we are now able to present a general definition of formalisms for representing knowledge in AI systems. 

\begin{defn}\label{defn:krf}
A {\em quasi knowledge representation formalism (qKRF)} over $(\mathscr{D},\mathscr{Q})$ is defined as a mapping $\Gamma$ such that 
\begin{enumerate}
\item $\textit{dom}(\Gamma)$ is recursive subset of $\{0,1\}^\ast$, and each string in $\textit{dom}(\Gamma)$ is called a {\em theory} of $\Gamma$;
\item $\Gamma$ maps each theory $\pi$ of $\Gamma$ to a KB over $(\mathscr{D},\mathscr{Q})$.
\end{enumerate}
Moreover, a qKRF $\Gamma$ is a {\em knowledge representation formalism (KRF)} if it admits an additional property as follows:
\begin{enumerate}
\item[3.] It is recursively enumerable to check, given $\pi\in \textit{dom}(\Gamma)$, $D\in\mathscr{D}$ and $\phi\in\mathscr{Q}$, whether $(D,\phi)\in\Gamma(\pi)$ or not.
\end{enumerate} 
\end{defn}

In the above definition, the three properties establish key requirements for every knowledge representation formalism:
Property 1 stipulates that the formalism must possess a language with an effective method for determining whether a given expression is legal; Property 2 defines the semantics of the formalism by associating each legal expression (or theory) in the language with an abstract knowledge base; and Property 3 ensures the implementability of the formalism by requiring that there exists a Turing machine capable of solving the query answering problem for this formalism.

\begin{exm}\label{lem:foso2db}
Suppose $\sigma_D$ involves no CWA-predicate symbols. Let $\mathscr{L}\in\{\mathrm{FO},\mathrm{SO}\}$ and $\sigma\supseteq(\sigma_D\cup\sigma_Q)\setminus\Delta$ be a signature. Let $\Gamma_{\mathscr{L}}$ be a mapping that maps each finite set $\Sigma$ of $\sigma$-sentences in $\mathscr{L}$ to a KB $K_\Sigma$ defined as follows:
$$K_\Sigma:=\{(D,\phi)\in\mathscr{D}\times\mathscr{Q}:D\cup\Sigma\vDash\phi\}.$$
It is easy to verify that both $\Gamma_{\mathrm{FO}}$ and $\Gamma_{\mathrm{SO}}$ are qKRFs over $(\mathscr{D},\mathscr{Q})$, and the former is a KRF, but the latter is not.
\end{exm}

\section{Universal KRFs}

In the last section, we have proposed a very general definition for knowledge representation formalisms. Both in theory and in practice, we would like the underlying knowledge representation formalism to be as expressive as possible. In this section, we will study what universal (q)KRFs are, and prove some interesting properties that such (q)KRFs enjoy.

There are at least two natural ways to define the class of universal (q)KRFs. The first is from a perspective on expressive power. We first present a notion defined in this way.

\begin{defn}
A qKRF $\Gamma$ over $(\mathscr{D},\mathscr{Q})$ is said to be {\em expressively complete} if $\textit{ran}(\Gamma)$ consists of all the recursively enumerable KBs over $(\mathscr{D},\mathscr{Q})$.
\end{defn}

According to the above definition, given an expressively complete qKRF $\Gamma$, for each knowledge base $K$ represented in $\Gamma$, there exists a Turing machine to implement query answering on $K$. However, this does not guarantee that $\Gamma$ itself qualifies as a KRF. To be a KRF, $\Gamma$ must have a single Turing machine capable of implementing query answering for all knowledge bases it represents. This leads to a natural question: Is there an expressively complete KRF? We will answer this question in the remainder of this section.

Another natural approach to defining universal KRFs involves reducibility between KRFs, defined as follows: 

\begin{defn}
Let $\Gamma$ and $\Gamma_0$ be KRFs over $(\mathscr{D},\mathscr{Q})$. Then $\Gamma$ is {\em reducible to} $\Gamma_0$ if there is a recursive function $p:\textit{dom}(\Gamma)\rightarrow\textit{dom}(\Gamma_0)$ such that $\Gamma=\Gamma_0\circ p$. In addition, we call $p$ a {\em reduction} from $\Gamma$ to $\Gamma_0$.
\end{defn}

With this notion, universal KRFs can be defined as below:

\begin{defn}\label{defn:univ_krf}
Let $\Gamma$ be a KRF over $(\mathscr{D},\mathscr{Q})$. Then $\Gamma$ is said to be {\em universal} if every KRF over $(\mathscr{D},\mathscr{Q})$ is reducible to $\Gamma$. 
\end{defn}

In other words, a universal KRF is a formalism such that query answering with any KB represented in a KRF can be implemented in the underlying formalism by an effective translation. Note that translating is a rather general approach to implement knowledge reasoning systems, and implementing knowledge reasoning systems by procedural (or other kinds of) programs is in fact a type of translations. These thus demonstrate why the above definition is natural for universal formalisms of knowledge representation.

In order for Definition~\ref{defn:univ_krf} to make sense, we have to answer the following question: {\em Is there indeed a universal KRF according to this definition?} 
At first glance, the answer to this question appears to be obviously affirmative. For example, as we know, every recursively enumerable KB can be accepted by a Turing machine. So, a naive idea is by defining a mapping $\Gamma$ as follows: If $M$ is Turing machine recognizing some KB $K$, then let $\Gamma(\enc{M}):=K$. Unfortunately, such a mapping is not possible to be a KRF because $\textit{dom}(\Gamma)$ is not recursive. Notice that the latter is an immediate corollary of Rice's theorem, see, e.g.,~\cite{Rogers67}.

To construct the desired KRF, our general idea is by carefully identifying a recursive subset $L$ of $\textit{dom}(\Gamma)$ such that the restriction of $\Gamma$ to $L$ is a KRF. To implement it, we propose an effective transformation to convert every Turing machine $M$ to a Turing machine $M^\ast$ that accepts a KB. The class of arbitrary Turing machines is clearly recursive. Since the transformation is effective, the class of Turing machines $M^\ast$, where $M$ is an arbitrary Turing machine, is thus recursive, too. This then implements the general idea.

Next, we show how to construct the transformation. Some notations are needed.
Given a Turing machine $M$, let 
$$\mathbb{K}(M):=\{(D,\phi)\in\mathscr{D}\times\mathscr{Q}:M\text{ accepts }\enc{D,\phi}\}.$$
Given a subclass $K$ of $\mathscr{D}\times\mathscr{Q}$, let $\textit{cl}(K)$ denote the minimum superclass of $K$ which admits Properties 1-5 of Definition~\ref{defn:kb}. {\em We need to assure the desired transformation $(\cdot)^\ast$ satisfying the property: $\mathbb{K}(M^\ast)=\textit{cl}(\mathbb{K}(M))$.}

The desired machine $M^\ast$ is constructed from $M$ by implementing Procedure~\ref{alg:rqa}. Now, let us explain how $M^\ast$ works. Let $D\in\mathscr{D}$ and $\phi\in\mathscr{Q}$. 
Roughly speaking, the computation of $M^\ast$ on the input $\enc{D,\phi}$ can be divided into six parts:
\begin{enumerate}
\item Simulate $M$ on $\enc{D,\phi}$; accept if $M$ accepts.
\item Check whether $\phi$ is a tautology; accept if true.
\item Check whether there is a sentence $\psi\in\mathscr{Q}$ such that $\psi\vDash\phi$ and that $M^\ast$ accepts $\enc{D,\psi}$; accept if true.
\item Check whether there exist a pair of sentences $\psi,\chi\in\mathscr{Q}$ such that $\phi=\psi\wedge\chi$ and that $M^\ast$ accepts both $\enc{D,\psi}$ and $\enc{D,\chi}$; accept if true.
\item Check whether there is a database $D'\in\mathscr{D}$ such that $D$ extends $D'$ and that $M^\ast$ accepts $\enc{D',\phi}$; accept if true.
\item Check whether there is a constant renaming (in fact, only need to consider injective functions from $C$ to $\Delta$ where $C$ is the set of constants appearing in either $\textit{DC}(D)$ or $\phi$) $\tau$ such that $M^\ast$ accepts $\enc{\tau(D),\tau(\phi)}$; accept if true.
\end{enumerate}

A direct implementation of the above procedure is generally impossible due to the following issues: Firstly, computations on Parts 1-6 may not terminate. Secondly, Part 3 needs to enumerate all sentences $\psi$ in $\mathscr{Q}$, and Part 6 involves generating all possible constant renamings for $D$, both enumerations are infinite. Thirdly, Parts 3-6 also involve recursive invocations of $M^\ast$, which makes the situation even worse.

To address these issues, we introduce a task array $T$ where the $i$-th task is stored in $T[i]$. Although there could be an infinite number of nonterminating tasks in $T$, a standard technique can be applied to sequentially simulate multiple such tasks. This can be visualized by a Cartesian coordinate system where the $x$-axis represents numbers of steps and the $y$-axis denotes task indices. The simulation is actually a procedure to traverse the first quadrant of this coordinate system, which is implemented by Lines 6-9 of Procedure~\ref{alg:rqa}.

An {\em atomic task} is a tuple $t := (N, p_1, \dots, p_k)$, where $N$ is a Turing machine and $p_1, \dots, p_k$ its parameters. Performing $t$ involves simulating $N$ on the input $\enc{p_1, \dots, p_k}$, with $t$ being {\em successful} if $N$ accepts.
Every {\em task} is either an atomic task or a finite sequence of atomic tasks $t:=\langle t_1, \dots, t_n \rangle$. Executing $t$ means sequentially performing the atomic tasks $t_1, \dots, t_n$. The task $t$ is said to be {\em successful} or to {\em succeed} if all its constituent atomic tasks are successful.

Moreover, let $M_e$ denote a Turing machine that accepts $\enc{\phi,\psi}$ iff $\phi,\psi$ are a pair of FO-sentences such that $\phi\vDash\psi$. Let $M_n$ denote a Turing machine that accepts $\enc{\tau}$ iff $\tau$ is a partial injective functions on $\Delta$ with a finite domain.  

As there might be an infinite number of tasks, we require $M^\ast$ to perform current tasks and generate new tasks at the same time. For example, in Part 1, we perform the first step of the task $(M,D,\phi)$ and add new tasks such as $(M_e,\top,\phi)$ to $T$ in the same {\bf for}-loop, see Lines 20-41 of Procedure~\ref{alg:rqa}.

Using the task array $T$, infinite enumerations can then be removed. For instance, to handle the infinite enumeration of queries in $\mathscr{Q}$ in Part 3, we start by letting $\psi$ be the first query (in the order of a natural encoding) in $\mathscr{Q}$, and add the task
$$t := \langle (M_e, \psi, \phi), (M, D, \psi) \rangle$$
to $T$. When $t$ is eventually executed, we add the task
$$\langle (M_e, \chi, \phi), (M, D, \chi) \rangle$$
to $T$, where $\chi$ is the next query after $\psi$ in $\mathscr{Q}$. By repeating this process, we effectively enumerate all queries in $\mathscr{Q}$.
For details, see Lines 22-23, 27-28, and 42-47 of Procedure~\ref{alg:rqa}.

Moreover, recursive invocations can be eliminated by utilizing the task array $T$, too. For instance, to compute $M^\ast$ on the input $\enc{D,\psi}$ in Part 3, we simply add the task $(M,D,\psi)$ to $T$. As per Lines 20-42 of Procedure~\ref{alg:rqa}, when $(M,D,\psi)$ is initiated, all tasks of computing the closure will be added to $T$ orderly to implement the computation of $(M^\ast,D,\psi)$.

A flag array $F$ records the completion status of each task as either {\bf true} or {\bf false}. The relationships between tasks are tracked using arrays $p$ (parent) and $c$ (cooperation). Specifically, $p[n]=i$ indicates that the task $T[n]$ is generated from the task $T[i]$, making $T[i]$ the parent of $T[n]$. A value of $c[i]=-1$ signifies that the task $T[i]$ operates independently. If $T[i]$ succeeds, its parent task's flag, $F[p[i]]$, is set to {\bf true}. Otherwise, if $c[i]=j \neq -1$, the task $T[i]$ must cooperate with the task $T[j]$. Only when both $T[i]$ and $T[j]$ succeed does the parent task's flag, $F[p[i]]$, get set to {\bf true}. For details on truth propagation, refer to Lines 14-16 of Procedure~\ref{alg:rqa}.

\begin{algorithm}[p]
\caption{The Workflow of $M^\ast$}\label{alg:rqa}
\label{alg:algorithm}
\KwIn{$D\in\mathscr{D}$ and $\phi\in\mathscr{Q}$}
\KwOut{Accept iff $(D,\phi)\in\textit{cl}(\mathbb{K}(M))$}
Initialize all positions in $F$ to \textbf{false};\\
Initialize all positions in $c$ to $-1$;\\
${T}[0]\leftarrow(M,D,\phi)$; \hfill \text{/* Set the root task */}\\
$p[0]\leftarrow -1$; \hfill \text{/* The root has no parent node */}\\
$n\leftarrow 1$; \hfill \text{/* There is one task in the current array */}\\
\For {$k\leftarrow0$ \KwTo $\infty$}
{
	\For {$i\leftarrow0$ \KwTo $\min(n-1,k)$}
	{
		\textrm{/* Simulate multiple tasks in parallel */}\\
		Perform the $(k-i)$-th step of ${T}[i]$ if it exists;\\
		\If {${T}[i]$ has just succeeded}{
			$F[i]\leftarrow\textbf{true}$;\\
			$j\leftarrow i$;\\
			\textrm{/* Propagate truth values upward */}\\
			\While {$j>0$ \& ($c[j]=-1$ or $F[c[j]]$)} 
			{
				$j\leftarrow p[j]$; \\
				$F[j]\leftarrow$ \textbf{true};
			}			
			\lIf {$F[0]$ \textrm{/* a support found */}} {accept}
		}
		\textrm{/* Generate tasks */}\\
		\If {$(M,D_0,\psi)$ in ${T}[i]$ has just started}
		{
			$T[n]\leftarrow(M_{e},\top,\psi)$;\hfill\textrm{/* Property 1 */}\\ 
 			$\chi\leftarrow$ the first query in $\mathscr{Q}$;\hfill\textrm{/* Property 2 */}\\  
 			$T[n+1]\leftarrow\langle(M_{e},\chi,\psi),(M,D_0,\chi)\rangle$;\\
			\textrm{/* Generate tasks for Property 5 */}\\
			$\tau\leftarrow$ the first constant renaming;\\
			$T[n+2]\leftarrow\langle(M_{n},\tau),(M,\tau(D_0),\tau(\psi))$;\\
			$p[n+2]\leftarrow p[n+1]\leftarrow p[n]\leftarrow i$;\\
			$n\leftarrow n+3$;\\
			\textrm{/* Generate tasks for Property 3 */}\\
			\If {$\psi=\chi\wedge\eta$ \& $\{\chi,\eta\}\subseteq\mathscr{Q}$}
			{
				$T[n]\leftarrow(M,D_0,\chi)$;\\
				$T[n+1]\leftarrow(M,D_0,\eta)$;\\
				$c[n]\!\leftarrow\! n+1$;\hfill\textrm{/*\! Coop.\! with $T[n\!+\!1]$\! */}\\
				$c[n+1]\leftarrow n$;\hfill\textrm{/*\! Coop.\! with $T[n]$\! */}\\
				$p[n]\leftarrow p[n+1]\leftarrow i$;\\
				$n\leftarrow n+2$;\\
			}
			\textrm{/* Generate tasks for Property 4 */}\\
			\ForAll {$D_1\in\mathscr{Q}$ s.t. $D_0$ extends $D_1$}
			{
 				$T[n]\leftarrow(M,D_1,\psi)$;\\
				$p[n]\leftarrow i$;\\
				$n\leftarrow n+1$;\\
			}
		}
		\ElseIf {${T}[i]=\langle(M_{e},\chi,\psi),(M,D_0,\chi)\rangle$ \\
		\quad\& $k-i=1$ \textrm{/* $T[i]$ has just started */}}
		{
			$\eta\leftarrow$ the query next to $\chi$ in $\mathscr{Q}$;\\
 			$T[n]\leftarrow\langle(M_{e},\eta,\psi),(M,D_0,\eta)\rangle$;\\
			$p[n]\leftarrow p[i]$;\\
			$n\leftarrow n+1$;\\
		}
		\ElseIf {${T}[i]=\langle(M_{n},\tau),(M,\tau(D_0),\tau(\psi))\rangle$ \\
		\quad\& $k-i=1$ \textrm{/* $T[i]$ has just started */}}
		{
			$\tau_1\leftarrow$ the constant naming next to $\tau$;\\
 			$T[n]\leftarrow\langle(M_{n},\tau_1),(M,\tau_1(D_0),\tau_1(\psi))\rangle$;\\
			$p[n]\leftarrow p[i]$;\\
			$n\leftarrow n+1$;\\
		}
	}
}
\end{algorithm}

\smallskip
According to the construction of $(\cdot)^\ast$, it is not difficult to prove the following proposition.

\begin{lem}\label{lem:star_kb}
$\mathbb{K}(M^\ast)=\textit{cl}(\mathbb{K}(M))$ for every Turing machine $M$.
\end{lem}

We are now in the position to define the desired KRF.

\begin{defn}
Let $\Theta$ denote the mapping which maps $\enc{M^\ast}$ to $\mathbb{K}(M^\ast)$ for every Turing machine $M$. 
\end{defn}

Does $\Theta$ serve as the required KRF? The subsequent theorem confirms this with a positive answer.

\begin{thm}\label{thm:theta_univ}
$\Theta$ is a universal KRF over $(\mathscr{D},\mathscr{Q})$. In addition, $\Theta$ is also expressively complete.
\end{thm}

\begin{proof}
We first show that $\Theta$ satisfies Conditions 1-3 of Definition~\ref{defn:krf}. Condition 2 immediately follows from Lemma~\ref{lem:star_kb}. Let $\mathcal{M}$ be the set that consists of $\enc{M}$ for all Turing machines $M$. Clearly, $\mathcal{M}$ is recursive. By the definition of $M^\ast$, we can effectively convert each $\enc{M}\in\mathcal{M}$ to $\enc{M^\ast}$. Consequently, $\textit{dom}(\Theta)$ is recursive, and we thus obtain Condition 1. Condition 3 is assured by the fact that there is a universal Turing machine which simulates each $M^\ast$ on any input $\enc{D,\phi}$, and this proves that $\Theta$ is indeed a KRF.

Next, let us consider the universality of $\Theta$. Let $\Gamma$ be an arbitrary KRF over $(\mathscr{D},\mathscr{Q})$. By definition, it is recursively enumerable to determine, given $\pi\in\textit{dom}(\Gamma)$, $D\in\mathscr{D}$ and $\phi\in\mathscr{Q}$, whether $(D,\phi)\in\Gamma(\pi)$ or not. Let $M$ be a Turing machine which solves that problem. Clearly, given any theory $\pi$ of $\Gamma$, one can construct a Turing machine $M_\pi$ that accepts the KB $\Gamma(\pi)$. Let $p$ be a function that maps $\pi$ to $\enc{M_\pi^\ast}$. It is easy to verify that $p$ is recursive and $\Gamma(\pi)=\Theta(p(\pi))$. Thus $\Gamma$ is reducible to $\Theta$, which yields the universality.

Now it remains to prove the expressive completeness. Let $K$ be any recursively enumerable KB. There is then a Turing machine $M$ such that, for all $D\in\mathscr{D}$ and $\phi\in\mathscr{Q}$, $(D,\phi)\in K$ iff $M$ accepts $\enc{D,\phi}$. By Lemma~\ref{lem:star_kb}, we thus have $$K=\textit{cl}(K)=\textit{cl}(\mathbb{K}(M))=\mathbb{K}(M^\ast)=\Theta(\enc{M^\ast}).$$
Since $K$ is arbitrary, $\Theta$ must be expressively complete.
\end{proof}

Interestingly, the reduction $p$ from any universal KRF to $\Theta$ given in the above proof is computable in linear time. This indicates that, despite $\Theta$ being a procedural KRF, {\em no (declarative) universal KRF can be linearly more succinct than $\Theta$}.

Thanks to the transitivity of reducibility, the following proposition is an immediate corollary of Theorem~\ref{thm:theta_univ}.

\begin{cor}
Every KRF $\Gamma$ over $(\mathscr{D},\mathscr{Q})$ is universal iff $\Theta$ is reducible to $\Gamma$. 
\end{cor}

With the above result, we thus call $\Theta$ the {\em canonical KRF}. One might wonder why we do not use a logical (or declarative) KRF as the canonical KRF. The main reasons are as follows: Firstly, proving the recursion theorem for a logical language appears to be challenging.
Secondly, and more importantly, while the logical language DED has been shown to be a universal KRF when databases contain only positive OWA-facts and queries are limited to CQs or UCQs~\cite{ZhangZY16,Zhang2020}, it remains an open question {\em whether there exist universal KRFs induced by natural logical languages for more general cases}.

Next, we present some interesting properties enjoyed by $\Theta$, which will play key roles in proving the main theorem. The following is a KRF-version of the padding lemma.

\begin{lem}\label{lem:padding}
For every theory $\pi$ of $\Theta$, we can effectively find an infinite set $S_{\pi}$ of theories of $\Theta$ such that $\Theta(\pi)=\Theta(\omega)$ for all theories $\omega\in S_{\pi}$.  
\end{lem}

The recursion theorem can also be generalized to KRFs.

\begin{thm}\label{thm:recursion}
Given any Turing machine that computes some recursive function $p:\textit{dom}(\Theta)\rightarrow \textit{dom}(\Theta)$, we can effectively find a theory $\pi$ of $\Theta$ such that $\Theta(p(\pi))=\Theta(\pi)$.
\end{thm}

\begin{proof}
Let $N_\Theta$ be a Turing machine that accepts $\enc{\pi,D,\phi}$ iff $(D,\phi)\in\Theta(\pi)$ for all $\pi\in\textit{dom}(\Theta)$, $D\in\mathscr{D}$ and $\phi\in\mathscr{Q}$.
Let $\mathcal{M}$ be a set consisting of $\enc{M}$ for all Turing machines $M$. Take $\kappa\in\mathcal{M}$ arbitrarily. We use $\varphi_\kappa$ to denote the function computed by the Turing machine encoded by $\kappa$. Now let us construct a Turing machine $M_\kappa$ which works as follows:
\begin{quote}
Given any input $\enc{D,\phi}$, first try to simulate the computation of $\varphi_\kappa$ on $\kappa$. If it halts with an output $\omega$, then simulate the computation of $N_\Theta$ on $\enc{\omega,D,\phi}$.
\end{quote} 
Let $M_\kappa^\ast$ be the Turing machine obtained from $M_\kappa$ by implementing Procedure~\ref{alg:algorithm}. Thus, we have 
$$\Theta(\enc{M^\ast_\kappa})=\left\{
\begin{aligned}
&\Theta(\varphi_\kappa(\kappa))\quad&&\text{if }\varphi_\kappa(\kappa)\text{ is defined;}\\
&\textit{cl}(\emptyset)&&\text{otherwise.}
\end{aligned}
\right.$$

Let $q$ be a mapping that maps $\kappa$ to $\enc{M_\kappa^\ast}$. Clearly, $q$ is recursive, which implies that $p\circ q$ is a recursive function from $\mathcal{M}$ to $\textit{dom}(\Theta)$. Let $M$ be a Turing machine that computes $p\circ q$, and let $\upsilon=\enc{M}$. It is easy to see that $\upsilon$ can be effectively obtained from $p$. Since $\varphi_\upsilon=p\circ q$ is recursive, we know that $\varphi_\upsilon(\upsilon)$ is defined. It is easy to verify that
$$\Theta(q(\upsilon))=\Theta(\enc{M^\ast_\upsilon})=\Theta(\varphi_\upsilon(\upsilon))=\Theta(p(q(\upsilon))).$$ 
Let $\pi=q(\upsilon)$. Clearly, $\Theta(\pi)=\Theta(p(\pi))$, and $\pi$ can be effectively obtained from $\upsilon$. These compete the proof.
\end{proof}

To present the main theorem, some notions are needed.

\begin{defn}
Let $\Gamma$ and $\Gamma_0$ be KRFs over $(\mathscr{D},\mathscr{Q})$, We say $\Gamma$ and $\Gamma_0$ are {\em recursively isomorphic} if there is a recursive bijection $p:\textit{dom}(\Gamma)\rightarrow \textit{dom}(\Gamma_0)$ such that $\Gamma=\Gamma_0\circ p$.
\end{defn}

We are now in the position to establish the main theorem.

\begin{thm}\label{thm:main1}
All universal KRFs over $(\mathscr{D},\mathscr{Q})$ are recursively isomorphic.
\end{thm}

\begin{proof}[Sketched Proof]
It suffices to show that every universal KRF $\Gamma$ is recursively isomorphic to $\Theta$. This is achieved by adopting a method used in the proof of Rogers's isomorphism theorem~\cite{Rogers58}. The main challenge is to find canonical universal KRFs, which we have resolved in Theorem~\ref{thm:theta_univ}. The rest of the proof proceeds as follows:  We find an injective reduction $p$ from $\Gamma$ to $\Theta$ via Lemma~\ref{lem:padding}, and an injective reduction $q$ from $\Theta$ to $\Gamma$ via Theorem~\ref{thm:recursion}. With $p$ and $q$, we construct the desired recursive isomorphism from $\Gamma$ to $\Theta$.
\end{proof}

\section{Subrecursive KRFs}

In the last section, we focused on universal KRFs. However, KRFs with low complexity, called {\em subrecursive KRFs}, might be useful in practice. A natural question thus arises as to under what condition subrecursive KRFs are recursively isomorphic.
We first present a candidate one as follows. 

\begin{defn}
We say a KRF $\Gamma$ {\em admits the padding property} if for each theory $\pi$ of $\Gamma$, we can effectively find an infinite set $S_\pi\subseteq\textit{dom}(\Gamma)$ such that $\Gamma(\pi)=\Gamma(\omega)$ for all $\omega\in S_\pi$. 
\end{defn}

By Lemma~\ref{lem:padding}, the canonical KRF $\Theta$ admits the padding property. Actually, the property holds for almost all the natural expressive formalisms. The following is an example.

\begin{exm}
Every FO-sentence $\phi$ is logical equivalent to $\phi\wedge\phi$. The defined KBs are thus the same, which provides a way to effectively find theories specified to the same KB. Thus, $\Gamma_{\mathrm{FO}}$ (see Example~\ref{lem:foso2db}) admits the padding property.
\end{exm}

Clearly, all recursively isomorphic KRFs should have the same expressive power. We adopt a slightly stronger condition that requires the KRFs here to be intertranslatable.

\begin{defn}
A pair of KRFs over $(\mathscr{D},\mathscr{Q})$, $\Gamma$ and $\Gamma_0$, are {\em equally strong} if $\Gamma$ is reducible to $\Gamma_0$ and vice versa.
\end{defn}

The main result for subrecursive KRFs is as follows. The proof mirrors that of Theorem~\ref{thm:main1}, with the only difference being the use of the padding property instead of Theorem~\ref{thm:recursion}.

\begin{thm}
All equally strong KRFs over $(\mathscr{D},\mathscr{Q})$ that admit the padding property are recursively isomorphic.
\end{thm}


\section{Conclusions and Related Work}

A general framework for studying KRFs has been proposed. Within the framework, all universal KRFs (respectively, all pairwise intertranslatable KRFs that admit the padding property) have been proven to be recursively isomorphic. 

Regarding the ``declarative vs. procedural controversy", our findings indicate that equally expressive natural KRFs exhibit similar performance such as reasoning efficiency. Moreover, no declarative KRF is linearly more succinct than the procedural KRF $\Theta$. Thus, labeling a KRF as declarative or procedural becomes less meaningful. Instead, declarativeness is better understood as a characteristic of specific representations rather than of the formalisms themselves.

For the debate between symbolic AI and connectionist AI, the existence of recursive isomorphisms between KRFs implies that for any knowledge operator (e.g., gradient descent) in a KRF we can effectively find an operator in another KRF to perform the same transformation. From a theoretical perspective, all these representation methodologies either pave the way to AGI or none, with core challenges being universal and advancements in one methodology benefiting others.

A closely related work is Rogers' isomorphism theorem, which states that all G\"{o}del numberings are recursively isomorphic~\cite{Rogers58}. This theorem plays a crucial role in computability theory. The proof of Theorem~\ref{thm:main1} involves an argument similar to that in~\cite{Rogers58}, but the challenges we encounter are quite different. While the existence of G\"{o}del numberings is readily apparent, establishing the existence of universal KRFs presents significant difficulties.

\newpage

\section*{Acknowledgements}

We would like to thank Professor Fangzhen Lin and anonymous referees for their helpful comments and suggestions. This work was supported by the Leading Innovation and Entrepreneurship Team of Zhejiang Province of China (Grant No. 2023R01008) and the Key Research and Development Program of Zhejiang, China (Grant No. 2024SSYS0012).

\bibliography{aaai25}

\comment{

\onecolumn

\newpage

\appendix
\section{Appendix: Detailed Proofs}

\bigskip
\subsection{Proofs in Section ``Framework"}

%
%

{\noindent\bf Proposition~\ref{prop:bm2kb}.} 
Let $\sigma$ be a signature such that $\sigma_D\cup\sigma_Q\subseteq\sigma$, and $\mathbb{M}$ be a belief mapping of $(\sigma_D,\sigma)$. Then $\textit{kb}(\mathbb{M},\mathscr{D},\mathscr{Q})$ is a KB over $(\mathscr{D},\mathscr{Q})$.

\begin{proof}
It is sufficient to prove that $\textit{kb}(\mathbb{M},\mathscr{D},\mathscr{Q})$ admits all of Properties 1-5 of Definition~\ref{defn:kb}.
Proofs of Properties 1-3 are trivial. Property 4 of Definition~\ref{defn:kb} follows from the third condition of Definition~\ref{defn:sm}, and Property 5 of Definition~\ref{defn:kb} follows from the second condition of Definition~\ref{defn:sm}. 
\end{proof}

{\noindent\bf Proposition~\ref{prop:kb_logic_impl}.} 
Let $K$ be a KB over $(\mathscr{D},\mathscr{Q})$, and $\sigma$ be the set consisting of all CWA-predicate symbols in $\sigma_D$. Then there are a set $\Sigma$ of FO-sentences such that, for all $D\in\mathscr{D}$ and $\phi\in\mathscr{Q}$, $\textit{Mod}^u_m(D,\Sigma,\sigma)\models\phi$ iff $(D,\phi)\in K$. 

\medskip

Before presenting the proof of this proposition, we first prove the compactness holds for first-order logic which satisfies the unique name assumption. Given a set $\Sigma$ of FO-sentences and an FO-sentence $\phi$, we write $\Sigma\vDash_{u}\phi$ if for all UNA-structures $\mathcal{A}$ we have $\mathcal{A}\models\phi$ if $\mathcal{A}\models\psi$ for all $\psi\in\Sigma$.
\smallskip

\begin{thm}\label{thm:compactness_una}
Let $\Sigma$ be a set of FO-sentences and $\phi$ an FO-sentence $\phi$ such that $\Sigma\vDash_u\phi$. Then there exists a finite subset $\Sigma_0$ of $\Sigma$ such that $\Sigma_0\vDash_u\phi$.
\end{thm}

\begin{proof}
To establish this, we will invoke Corollary 4.1.11 from~\cite{CK1990}, which is stated as follows:
\medskip

{\noindent\em Claim 1}~\cite{CK1990}. 
Let $\Phi$ be a set of sentences of a signature $\sigma$, let $I$ be the set of all finite subsets of $\Phi$, and for each $i\in I$, let $\mathcal{A}_i$ be a model of $i$. Then there exists an ultrafilter $D$ over $I$ such that the ultraproduct $\prod_{D}\mathcal{A}_i$ is a model of $\Sigma$.
\medskip

For details of notions and notations about ultraproducts, please refer to Chapter 4 of~\cite{CK1990}.\medskip

To make Claim 1 applicable to first-order logic under the unique name assumption, we prove a lemma as follows:
\medskip

{\noindent\em Claim 2.} If $(\mathcal{A}_i)_{i\in I}$ are UNA-structures and $D$ a ultrafilter over $I$, then $\prod_{D}\mathcal{A}_i$ is also a UNA-structure. 

\begin{proof}
Let $\mathcal{B}:=\prod_{D}\mathcal{A}_i$. Take $c,d$ as an arbitrary pair of distinct constants in $\sigma$. It suffices to prove that $c^{\mathcal{B}}\ne d^{\mathcal{B}}$. Clearly,
$$c^{\mathcal{B}}=\langle c^{\mathcal{A}_i}:i\in I\rangle_D\text{\quad and\quad}d^{\mathcal{B}}=\langle d^{\mathcal{A}_i}:i\in I\rangle_D.$$
Since $(\mathcal{A}_i)_{i\in I}$ are UNA-structures, we conclude that $\{i\in I:c^{\mathcal{A}_i}=d^{\mathcal{A}_i}\}=\emptyset$. According to the assumption, $D$ is an ultrafilter, we know that $\emptyset\not\in D$, which implies that $c^{\mathcal{B}}\ne d^{\mathcal{B}}$ as desired.
\end{proof}

Now let us show the desired theorem. Towards a contradiction, assume $\Sigma_0\not\vDash_u\phi$ for every subset set $\Sigma_0$ of $\Sigma$. Let $\Phi:=\Sigma\cup\{\neg\phi\}$ and $\Phi_0:=\Sigma_0\cup\{\neg\phi\}$. Then, $\Phi_0$ must be satisfied by some UNA-structure of $\sigma$. Let $\mathcal{A}_{\Phi_0}$ denote such a structure. Let $I$ denote the set of all finite subsets of $\Phi$. By Claim 1, we know that there is an ultrafilter $D$ over $I$ such that $\prod_D\mathcal{A}_i$ is a model of $\Phi$. According to Claim 2, $\prod_D\mathcal{A}_i$ is a UNA-structure. Consequently, we have $\Sigma\not\vDash_u\phi$, a contradiction as desired.
\end{proof}

Now we are able to prove Proposition~\ref{prop:kb_logic_impl}.

\begin{proof}[Proof of Proposition~\ref{prop:kb_logic_impl}]
Some notations are needed. Let $D\in\mathscr{D}$. For $t\in\{0,1\}$, let $A_t$ denote the set of all atoms $\alpha\in\textit{dom}(D)$ such that $D(\alpha)=t$. Let $\theta_{D}$ be a conjunction of all sentences in $A_1\cup\{\neg\alpha:\alpha\in A_0\}$. 
For each constant $c\in\Delta$, we introduce $v_c$ as a fresh variable. For each pair $(D,\phi)\in K$, let
$$\gamma_{D,\phi}:=\forall\bar{v}(\theta_D^v\wedge\textit{PDst}(\bar{v})\rightarrow\phi^v)$$
where $\bar{v}$ denotes a tuple that consists of all the variables appearing in $\theta_D^v$, $\textit{PDst}(\bar{v})$ is a formula asserting variables in $\bar{v}$ are pairwise distinct, and for $\psi\in\{\theta_D,\phi\}$, $\psi^v$ denotes the formula obtained from $\psi$ by substituting $v_c$ for every occurrence of each $c\in\Delta$. Let $\Sigma$ be the set that consists of $\gamma_{D,\phi}$ for all pairs $(D,\phi)\in K$. Now, it remains to prove that, for all $D\in\mathscr{D}$ and $\phi\in\mathscr{Q}$, $\textit{Mod}^u_m(D,\Sigma,\sigma)\models\phi$ iff $(D,\phi)\in K$. 

First, let us consider the ``if" direction. Suppose $(D,\phi)\in K$. Let $\mathcal{A}$ be a UNA-structure of $\sigma_D\cup\sigma_Q$ that is a $\subseteq_\sigma$-minimal model of both $D$ and $\Sigma$. We need to prove that $\mathcal{A}$ is a model of $\phi$. Let $s$ be an assignment in $\mathcal{A}$ that maps $v_c$ to $c^{\mathcal{A}}$ for all $c\in\Delta$. As $\mathcal{A}$ satisfies the unique name assumption, we know that $s$ is injective, which implies $\mathcal{A}\models\textit{PDst}(\bar{v})[s]$. From $\mathcal{A}$ is a model of $D$, we also have $\mathcal{A}\models\theta_{D}^v[s]$. By the construction of $\Sigma$ we know $\gamma_{D,\phi}\in\Sigma$, which implies that $\mathcal{A}$ is a model of $\gamma_{D,\phi}$.  Consequently, it holds that $\mathcal{A}\models\phi^v[s]$. The latter is equivalent to $\mathcal{A}\models\phi$ that we need to prove.

Next, let us consider the ``only-if" direction. Assume $\textit{Mod}^u_m(D,\Sigma,\sigma)\models\phi$. Our task is to prove $(D,\phi)\in K$. It is trivial for the case where $\phi$ is a tautology. For the case where $\phi$ is not a tautology, let $\Psi\subseteq\mathscr{Q}$ denote the set of all sentences $\psi\in\mathscr{Q}$ which satisfy the following property: 
\begin{quote}
There exists at least one database $D_0\in\mathscr{D}$ such that $D$ extends $D_0$ and $(D_0,\psi)\in K$. 
\end{quote}
We need to prove $\Psi\vDash_u\phi$. Let $\mathcal{A}$ be a UNA-structure of $\sigma_Q$ that satisfies $\Psi$, and $\mathcal{B}$ a UNA-structure of $\sigma_D\cup\sigma_Q$ that is both a $\subseteq_\sigma$-minimal model of $D$ and an expansion of $\mathcal{A}$. Such a model always exists because $\sigma_D\cap\sigma_Q$ contains nothing but constants, and no predicate symbol in $\sigma_Q$ appears in $\sigma$. We need to prove $\mathcal{B}\models\Sigma$. 

Take $\gamma_{D_0,\psi}\in\Sigma$ arbitrarily. If there exists no injective assignment $s$ in $\mathcal{B}$ such that $\mathcal{B}\models\theta_{D_0}^v[s]$, then $\mathcal{B}$ is trivially a model of $\gamma_{D_0,\psi}$. Otherwise, let $s$ be any of such assignments. Let $\tau$ be a mapping that maps each $c\in\Delta$ to $s(v_c)$. It is not difficult to verify that $\tau$ is injective and $D$ extends $\tau(D_0)$. By the construction of $\Sigma$ and $\gamma_{D_0,\psi}\in\Sigma$, we know $(D_0,\psi)\in K$. Since $\mathscr{Q}$ is closed under constant renaming, it must be true that $\tau(\psi)\in\mathscr{Q}$. According to Property 5 of Definition~\ref{defn:kb}, we thus have $(\tau(D_0),\tau(\psi))\in K$, and consequently, $\tau(\psi)\in\Psi$. This implies that $\mathcal{B}$ is a model of $\tau(\psi)$, or equivalently, $\mathcal{B}\models\psi[s]$. As a consequence, $\mathcal{B}$ is also a model of $\gamma_{D_0,\psi}$ in this case. By the arbitrariness of $\gamma_{D_0,\psi}$, we then conclude that $\mathcal{B}$ is a model of $\Sigma$. By definition, it is easy to see that $\mathcal{B}|_{\sigma_Q}=\mathcal{A}$. Consequently, we have $\mathcal{A}\models\phi$. This thus proves $\Psi\vDash_u\phi$ that we need.

According to Theorem~\ref{thm:compactness_una}, there is a finite subset $\Psi_0$ of $\Psi$ such that $\Psi_0\vDash_u\phi$. For each $\psi\in\Psi_0$, by Property 4 of Definition~\ref{defn:kb}, we conclude $(D,\psi)\in K$. Let $\chi$ denote the conjunction of all sentences $\psi\in\Psi_0$. As $\mathscr{Q}$ is closed under conjunctions, it must be true that $\chi\in\mathscr{Q}$. By applying Property 3 of Definition~\ref{defn:kb} a finite number of times, we thus have $(D,\chi)\in K$. It is also clear that $\chi\vDash\phi$. According to Property 2 of Definition~\ref{defn:kb}, we then obtain $(D,\phi)\in K$, which completes the proof.
\end{proof}

\subsection{Proofs in Section ``Universal KRFs"}

\medskip
{\noindent\bf Lemma~\ref{lem:star_kb}.} 
$\mathbb{K}(M^\ast)=\textit{cl}(\mathbb{K}(M))$ for every Turing machine $M$.

\begin{proof}
For convenience, we say a task in $T$ is {\em render} as {\bf true} if the flag in $F$ corresponding to the task is set to {\bf true}.

We first consider the direction of ``$\subseteq$". Let $(D,\phi)\in\mathbb{K}(M^\ast)$. Now our task is to prove that $(D,\phi)\in\textit{cl}(\mathbb{K}(M))$. By definition, $M^\ast$ will accept the input $\enc{D,\phi}$. Let $N(D,\phi)$ denote the number of iterations of the {\bf while}-loop (see Lines 14-16 of Procedure~\ref{alg:rqa}) before the task in which $(M,D,\phi)$ appears is rendered as {\bf true}. 

Now we prove the desired conclusion by a routine induction on $N(D,\phi)$. If $N(D,\phi)=0$, we have $(D,\phi)\in\mathbb{K}(M)$, which yields $(D,\phi)\in\textit{cl}(\mathbb{K}(M))$ immediately. For the case where $N(D,\phi)>0$, assume as the inductive hypothesis that 
$$
\left\{
\begin{aligned}
&(D_0,\psi)\in\mathscr{D}\times\mathscr{Q}\,\,\&\\
&N(D_0,\psi)<N(D,\phi)
\end{aligned}
\right\} \Longrightarrow(D_0,\psi)\in\textit{cl}(\mathbb{K}(M)).
$$
We need to prove $(D,\phi)\in\textit{cl}(\mathbb{K}(M))$. By Procedure~\ref{alg:rqa}, it is easy to see that at least one of the following cases must be true:
\begin{enumerate}
\item $(M,D,\phi)$ has a child task $(M_e,\top,\phi)$ rendered as {\bf true};
\item $(M,D,\phi)$ has a child task $\langle(M_e,\psi,\phi),(M,D,\psi)\rangle$ rendered as {\bf true} such that $\psi\in\mathscr{Q}$ and $\psi\vDash\phi$;
\item $(M,D,\phi)$ has child tasks $(M,D,\psi)$ and $(M,D,\chi)$ rendered as {\bf true} such that $\psi,\chi\in\mathscr{Q}$ and $\phi=\psi\wedge\chi$;
\item $(M,D,\phi)$ has a child task $(M,D_0,\phi)$ rendered as {\bf true} such that $D_0\in\mathscr{D}$ and $D_0$ extends $D$;
\item $(M,D,\phi)$ has a child task $(M,\tau(D),\tau(\phi))$ rendered as {\bf true} such that $\tau:\Delta\rightarrow\Delta$ is injective.
\end{enumerate}

We only consider Cases 2 and 3. Proofs for other cases are similar. For Case 2, it is clear that $N(D,\phi)=N(D,\psi)+1$. By the inductive hypothesis, we thus have $(D,\psi)\in\textit{cl}(\mathbb{K}(M))$. Since $\psi\vDash\phi$, by the definition of $\textit{cl}(\cdot)$ we obtain $(D,\phi)\in\textit{cl}(\mathbb{K}(M))$ immediately.
For Case 3, it is also easy to see that both $N(D,\psi)<N(D,\phi)$ and $N(D,\chi)<N(D,\phi)$ hold. Applying the inductive hypothesis, we then have both $(D,\psi)\in\textit{cl}(\mathbb{K}(M))$ and $(D,\chi)\in\textit{cl}(\mathbb{K}(M))$. Since $\phi=\psi\wedge\chi$, by the definition of $\textit{cl}(\cdot)$ we obtain $(D,\phi)\in\textit{cl}(\mathbb{K}(M))$ immediately.

Next, we prove the direction of ``$\supseteq$". Suppose $(D,\phi)\in\textit{cl}(\mathbb{K}(M))$. By definition, there are a sequence $S$ of pairs $$(D_0,\phi_0),(D_1,\phi_1),\dots,(D_n,\phi_n)$$
such that $D=D_n$ and $\phi=\phi_n$ and for each $i\in\{0,1,\dots,n\}$, at least one of the following cases must be true:
\begin{enumerate}
\item $M$ accepts $\enc{D_i,\phi_i}$;
\item $\phi_i$ is a tautology;
\item there is an integer $j\in\{0,\dots,i-1\}$ such that $D_i=D_j$ and $\phi_j\vDash\phi_i$;
\item there are a pair of integers $j,k\in\{0,\dots,i-1\}$ such that $D_i=D_j=D_k$ and $\phi_i=\phi_j\wedge\phi_k$;
\item there is an integer $j\in\{0,\dots,i-1\}$ such that $D_j$ extends $D_i$ and $\phi_j=\phi_i$;
\item there is an integer $j\in\{0,\dots,i-1\}$ and an injection $\tau:\Delta\rightarrow\Delta$ such that $D_j=\tau(D_i)$ and $\phi_j=\tau(\phi_i)$.
\end{enumerate} 
The sequence $S$ thus provides a path for how $M^\ast$ halts and accepts $\enc{D,\phi}$, or more explicitly, how the truth values in $F$ are propagated. By Procedure~\ref{alg:rqa}, for every $i\in\{0,1,\dots,n\}$, there must be a task, denoted $t_i$, in the task array $T$ such that $t_i$ contains or is exactly the task $(M,D_i,\phi_i)$. Now we claim that each of such tasks will be rendered as {\bf true} in a finite time. 

We prove this by a routine induction on the length $n$ of $S$. For the case of $n=0$, $S$ consists of exactly one pair $(D,\phi)$, and we have that either
 $M$ accepts $\enc{D,\phi}$ or
 $\phi$ is a tautology. Thus, $t_0=(M,D,\phi)$ if the first case happens, and $t_0=(M_e,\top,\phi)$ otherwise.
It is easy to see that in either case performing $t_0$ can be done in a finite time. (Note that the validity of the query language $\mathscr{Q}$ is recursively enumerable.) According to Lines 11-16 of Procedure~\ref{alg:rqa}, $t_0$ will be rendered as {\bf true} in a finite time.

Let $n>0$. As the inductive hypothesis, we assume that, for every $i\in\{0,\dots,n-1\}$, the task $t_i$ has been rendered as {\bf true}. We need to show that the task $t_n$ will also be rendered as {\bf true}. Only consider Case 3, i.e., the case where $D=D_j$ and $\phi_j\vDash\phi$ for some $j\in\{0,\dots,n-1\}$. Proofs for other cases are similar, and we omit them here. For Case 3, it is easy to see that $t_j=\langle(M_e,\phi_j,\phi),(M,D_j,\phi_j)$ and $t_n$ is the parent task of $t_j$ according to Procedure~\ref{alg:rqa}. By the inductive hypothesis, the task $t_j$ will be rendered as {\bf true} in a finite time. Through the truth propagation of $F$ (see Lines 14-16 of Procedure~\ref{alg:rqa}), we can conclude that $t_n$ must also be rendered as {\bf true} in a finite time.  

With the above conclusion, we then know that the root task $(M,D,\phi)$ will be rendered as {\bf true} in a finite time, and $M^\ast$ thus accepts $\enc{D,\phi}$. This completes the proof.  
\end{proof}

{\noindent\bf Lemma~\ref{lem:padding}.} 
For every theory $\pi$ of $\Theta$, we can effectively find an infinite set $S_{\pi}$ of theories of $\Theta$ such that $\Theta(\pi)=\Theta(\omega)$ for all theories $\omega\in S_{\pi}$.

\begin{proof}
Let $\pi\in\textit{dom}(\Theta)$. According to the construction of $\Theta$, we know that there is a Turing machine $M$ such that $\pi=\enc{M^\ast}$. Suppose $s_0,s_1,\dots,s_n$ list all the states of $M$. We introduce $s_{n+1},s_{n+2},\dots$ as a countably infinite number of (pairwise distinct) fresh states. Take $k\ge 1$, and let $M_k$ be a Turing machine obtained from $M$ by adding the following instructions: $$\langle s_{n+1},B,s_{n+1},B,R\rangle,\dots, \langle s_{n+k},B,s_{n+k},B,R\rangle.$$
I.e., for $i\in\{1,\dots,k\}$, the $i$-th instruction above indicates that if the state of $M_k$ is $s_{n+i}$ and the symbol in the scanned cell is $B$, then both the state and the symbol in the scanned cell will not change, and the reading head will move right one cell.  
Let $\omega_k:=\enc{M_k^\ast}$. It is clear that $\Theta(\omega_k)=\Theta(\pi)$, and $\omega_k$ can be effectively obtained from $\pi$. It is also easy to see that, by employing a standard encoding technique,  $\omega_1,\omega_2,\dots$ are pairwise distinct. These thus yield the lemma.
\end{proof}

{\noindent\bf Theorem~\ref{thm:main1}.} 
All the universal KRFs over $(\mathscr{D},\mathscr{Q})$ are recursively isomorphic.
\medskip

Let $\Gamma$ be an arbitrary universal KRF over $(\mathscr{D},\mathscr{Q})$. To prove the above theorem, it suffices to prove that $\Gamma$ and the canonical KRF $\Theta$ are recursively isomorphic. Before presenting the proof, we first prove some lemmas.

\begin{lem}\label{lem:gamma2theta}
There is an injective reduction from $\Gamma$ to $\Theta$.
\end{lem}

\begin{proof}
Note that $\Theta$ is a universal KRF over $(\mathscr{D},\mathscr{Q})$. By definition, $\Gamma$ must be reducible to $\Theta$. Let $p$ be a reduction from $\Gamma$ to $\Theta$. We need to find an injective reduction $h$ from $\Gamma$ to $\Theta$. Suppose $\pi_1,\pi_2,\dots$ is an effective enumeration of all theories of $\Gamma$. Since $\textit{dom}(\Gamma)$ is recursive, such an enumeration always exists. We attempt to construct a sequence of mappings $h_i:\textit{dom}(\Gamma)\rightharpoonup \textit{dom}(\Theta)$, $i\in\mathbb{N}$ (where $\mathbb{N}$ denoted the set of natural numbers), such that, for all $i\in\mathbb{N}$, the following conditions hold:
\begin{enumerate}
\item $h_i$ is injective and $dom(h_i)=\{\pi_1,\pi_2\dots,\pi_i\}$;
\item $\Gamma(\pi)=\Theta(h_i(\pi))$ whenever $\pi\in\textit{dom}(h_i)$.
\end{enumerate} 

Let $h_0:=\emptyset$. Take $i\in\mathbb{N}$, and suppose $h_i$ is a mapping satisfying the above conditions. We define $h_{i+1}$ by cases as follows:
\begin{enumerate}
\item If $p(\pi_{i+1})\not\in\textit{ran}(h_i)$, then let 
$$h_{i+1}:=h_i\cup\{\pi_{i+1}\mapsto p(\pi_{i+1})\}.$$
It is easy to verify that $h_{i+1}$ satisfies Conditions 1-2, and $h_{i+1}$ is thus desired.
\item Otherwise, we must have $p(\pi_{i+1})\in\textit{ran}(h_i)$. By applying Lemma~\ref{lem:padding}, we can effectively find a sequence of distinct theories $\omega_1,\omega_2,\dots,\omega_{i+1}$ of $\Theta$ such that 
$$\Theta(p(\pi_{i+1}))=\Theta(\omega_0)=\Theta(\omega_1)=\cdots=\Theta(\omega_{i+1}).$$
Let $S:=\{\omega_1,\omega_2,\dots,\omega_{i+1}\}\setminus\textit{ran}(h_i)$. It is easy to see that $S\ne\emptyset$. Let $\omega_j$ be a theory in $S$ with the least index $j$, and let 
$$h_{i+1}:=h_i\cup\{\pi_{i+1}\mapsto \omega_{j}\}.$$
In this case, $h_{i+1}$ also satisfies Conditions 1-2.
\end{enumerate}
Let $h=\bigcup_{i\in\mathbb{N}}h_i$. One can easily verify that $h$ is an injective reduction from $\Gamma$ to $\Theta$, which completes the proof.
\end{proof}

\begin{lem}\label{lem:theta2gamma}
There is an injective reduction from $\Theta$ to $\Gamma$.
\end{lem}

\begin{proof}
We prove this by implementing a construction similar to that for Lemma~\ref{lem:gamma2theta}. The main difficulty here is that the padding lemma may not hold for $\Gamma$. To overcome the difficulty, we use the method in~\cite{Rogers58}, by employing an effective version of the recursion theorem for the canonical KRF $\Theta$, i.e., Theorem~\ref{thm:recursion}. 

Let $p$ be a reduction from $\Theta$ to $\Gamma$. Since $\Gamma$ is a universal KRF over $(\mathscr{D},\mathscr{Q})$, by definition such a reduction always exists. To make the proof of Lemma~\ref{lem:gamma2theta} works here, it suffices to devise an effective procedure to solve the following problem: Given any finite sequence of theories $\pi_0,\pi_1,\dots,\pi_k$ of $\Theta$ such that 
$$\Theta(\pi_0)=\Theta(\pi_1)=\cdots=\Theta(\pi_k),$$
find a theory $\pi_{k+1}$ that satisfies the following conditions: 
\begin{enumerate}
\item $\Theta(\pi_{k+1})=\Theta(\pi_0)$ and 
\item $p(\pi_{k+1})\not\in\{p(\pi_0),p(\pi_1),\dots,p(\pi_k)\}$.
\end{enumerate}

Now let us present the desired procedure. For any theory $\omega$ of $\Theta$, we construct a Turing machine $M_\omega$ to implement the following computation:
\begin{quote}
Given any $\enc{D,\phi}$ (where $D\in\mathscr{D}$ and $\phi\in\mathscr{Q}$) as input, first check whether $p(\omega)\not\in\{p(\pi_0),p(\pi_1),\dots,p(\pi_k)\}$. If this is true, then check whether $(D,\phi)\in\Theta(\pi_0)$; otherwise, never stop.
\end{quote}
By the definition of KRF, we know that there is a Turing machine to determine whether $(D,\phi)\in\Theta(\pi_0)$ for any proper $D,\phi$ and $\pi_0$. As $p$ is recursive, there is also a Turing machine to check whether $p(\pi_{k+1})\not\in\{p(\pi_0),p(\pi_1),\dots,p(\pi_k)\}$. Thus, $M_\omega$ can be effectively constructed from $\pi_0,\pi_1,\dots,\pi_k$ and $\omega$. Let $M_\omega^\ast$ be a Turing machine to implement Procedure~\ref{alg:rqa} based on $M_\omega$. 

Let $q$ be a function that maps every theory $\pi$ of $\Theta$ to $\enc{M_\pi^\ast}$. Clearly, $q$ is a recursive function, and we have 
$$\Theta(q(\omega))=\left\{
\begin{aligned}
&\Theta(\pi_0)&&\text{if }p(\omega)\not\in\{p(\pi_0),p(\pi_1),\dots,p(\pi_k)\};\\
&\textit{cl}(\emptyset)&&\text{otherwise}.
\end{aligned}
\right.$$

Let $M^q$ be a Turing machine that computes $q$. By Theorem~\ref{thm:recursion}, given $\enc{M^q}$ as input, we can effectively find a theory $\upsilon$ of $\Theta$ such that $\Theta(q(\upsilon))=\Theta(\upsilon)$. We thus have
$$\Theta(\upsilon)=\left\{
\begin{aligned}
&\Theta(\pi_0)&&\text{if }p(\upsilon)\not\in\{p(\pi_0),p(\pi_1),\dots,p(\pi_k)\};\\
&\textit{cl}(\emptyset)&&\text{otherwise}.
\end{aligned}
\right.$$

If $p(\upsilon)\not\in\{p(\pi_0),p(\pi_1),\dots,p(\pi_k)\}$, we let $\pi_{k+1}:=\upsilon$, and $\pi_{k+1}$ satisfies the aforementioned Conditions 1-2 in this case.

Otherwise, there must be some $i\in\{0,\dots,k\}$ such that $p(\pi_i)=p(\upsilon)$. As $p$ is a reduction from $\Theta$ to $\Gamma$, we thus have
$$\Theta(\pi_0)=\Theta(\pi_1)=\cdots=\Theta(\pi_k)=\Theta(\upsilon)=\textit{cl}(\emptyset).$$

Let $\omega$ be any theory of $\Theta$. We construct a Turing machine $N_\omega$ from $\omega$ to implement the following computation:
\begin{quote}
Given any $\enc{D,\phi}$ (where $D\in\mathscr{D}$ and $\phi\in\mathscr{Q}$) as input, first check whether $p(\omega)\in\{p(\pi_0),p(\pi_1),\dots,p(\pi_k)\}$. If this is true, then accept; otherwise, never stop.
\end{quote}
In addition, let $N_\omega^\ast$ be a Turing machine to implement Procedure~\ref{alg:rqa} based on $N_\omega$, and let $h$ be a function that maps each theory $\omega\in\textit{dom}(\Theta)$ to $\enc{N_\omega^\ast}$. It is also easy to see that $h$ is recursive, and
$$\Theta(h(\omega))=\left\{
\begin{aligned}
&\mathscr{D}\times\mathscr{Q}\!\!&&\text{if }p(\omega)\in\{p(\pi_0),p(\pi_1),\dots,p(\pi_k)\};\\
&\textit{cl}(\emptyset)&&\text{otherwise}.
\end{aligned}
\right.$$

Let $M^h$ denote a Turing machine that computes $h$. Applying Theorem~\ref{thm:recursion} again, one can effectively find a theory $\kappa$ of $\Theta$ such that $\Theta(h(\kappa))=\Theta(\kappa)$. We thus have 
$$\Theta(\kappa)=\left\{
\begin{aligned}
&\mathscr{D}\times\mathscr{Q}&&\text{if }p(\kappa)\in\{p(\pi_0),p(\pi_1),\dots,p(\pi_k)\};\\
&\textit{cl}(\emptyset)&&\text{otherwise}.
\end{aligned}
\right.$$

In this case, we claim that $p(\kappa)\not\in\{p(\pi_0),p(\pi_1),\dots,p(\pi_k)\}$. Otherwise, by the above equation, we have $\Theta(\kappa)=\mathscr{D}\times\mathscr{Q}$. On the other hand, there must some $\pi_i$ such that $p(\pi_i)=p(\kappa)$, which implies $\Theta(\kappa)=\textit{cl}(\emptyset)$. Since $\mathscr{Q}$ contains at least one non-tautological sentence, we know $\textit{cl}(\emptyset)\ne\mathscr{D}\times\mathscr{Q}$. A contradiction is thus obtained from the above conclusions. 

Now, let $\pi_{k+1}:=\kappa$. It is easy to see that $\pi_{k+1}$ satisfies Conditions 1-2 in this case, which completes the proof.
\end{proof}

\begin{lem}\label{lem:injective2bijective}
Suppose $p$ is an injective reduction from $\Gamma$ to $\Theta$, and $q$ an injective reduction from $\Theta$ to $\Gamma$. Then $\Gamma$ and $\Theta$ are recursively isomorphic.
\end{lem}

\begin{proof}
Let $\pi_0,\pi_1,\dots$ be an effective enumeration of all theories of $\Gamma$, and $\omega_0,\omega_1,\dots$ an effective enumeration of all theories of $\Theta$. Since both $\textit{dom}(\Gamma)$ and $\textit{dom}(\Theta)$ are recursive, such enumerations always exist.
We attempt to construct a sequence of mappings $h_i:\textit{dom}(\Gamma)\rightharpoonup \textit{dom}(\Theta)$, $i\in\mathbb{N}$, such that, for all $i\in\mathbb{N}$, the following conditions hold:
\begin{enumerate}
\item $h_i$ is injective;
\item $\Gamma(\pi)=\Theta(h_i(\pi))$ whenever $\pi\in\textit{dom}(h_i)$.
\end{enumerate} 

Let $h_0:=\emptyset$. Suppose $n\ge 0$. We define $h_{n+1}$ by cases as follows:
\begin{enumerate}
\item $n=2k$: If $\pi_k\in\textit{dom}(h_n)$, we just let $h_{n+1}:=h_{n}$. Otherwise, we check whether $p(\pi_k)\in\textit{ran}(h_{n})$. If no, let 
$$h_{n+1}:=h_{n}\cup\{\pi_k\mapsto p(\pi_k)\}.$$
Otherwise, let $S$ denote the following sequence:
$$\pi_k, p(\pi_k),h_n^{-1}(p(\pi_k)),p(h_n^{-1}(p(\pi_k))),\dots$$
where elements in the sequence are obtained by alternately applying $p$ and $h_n^{-1}$ until no new elements can be generated.

\medskip
We claim: 

\medskip
{\noindent\em Claim 1.} There is at least one theory $\omega$ of $\Gamma$ in $S$ such that $\omega\not\in\textit{ran}(h_n)$. 
\medskip

{\em Proof.} As both $p$ and $h^{-1}_n$ are injective, every theory in $S$ has at most one predecessor and at most one successor. In addition, since $\pi_k\not\in\textit{dom}(h_n)$, $S$ must be a chain. According to the construction of $h_n$, we know that $\textit{ran}(h_n)$ has only a finite number of elements, which implies that $S$ is a finite chain. Consequently, the last element of $S$ is a theory $\omega$ of $\Gamma$ on which $h^{-1}$ is undefined. (Otherwise, the chain can be extended by $h^{-1}$, a contradiction.) This then proves Claim 1. \hfill$\square$\!\!
\medskip

With Claim 1, we let $$h_{n+1}:=h_n\cup\{\pi_k\mapsto\omega\}.$$

\item $n=2k+1$: If $\kappa_k\in\textit{ran}(h_{n})$, we just let $h_{n+1}:=h_{n}$. Otherwise, we check whether $q(\kappa_k)\in\textit{dom}(h_{n})$. If no, let 
$$h_{n+1}:=h_{n}\cup\{q(\kappa_k)\mapsto\kappa_k\}.$$ 
Otherwise, let $T$ denote the following sequence:
$$\kappa_k, q(\kappa_k),h_n(q(\kappa_k)),q(h_n(q(\kappa_k))),\dots$$
where elements in the sequence are obtained by alternately applying $q$ and $h_n$ until no new elements can be generated.

\medskip
We claim:

\medskip
{\noindent\em Claim 2.} There is at least one theory $\omega$ of $\Theta$ in $S'$ such that $\omega\not\in\textit{dom}(h_n)$. 
\medskip

{\em Proof.} As both $q$ and $h_n$ are injective, every theory in $S'$ has at most one predecessor and at most one successor. In addition, since $\kappa_k\not\in\textit{ran}(h_n)$, $S'$ must be a chain. According to the construction of $h_n$, we know that $\textit{dom}(h_n)$ has only a finite number of elements, which implies that $S'$ is a finite chain. Consequently, the last element of $S'$ is a theory $\omega$ of $\Theta$ on which $h$ is undefined. (Otherwise, the chain can be extended by $h$, a contradiction.) This then proves Claim 2. \hfill$\square$\!\!
\medskip

With Claim 2, we let $$h_{n+1}:=h_n\cup\{\omega\mapsto\kappa_k\}.$$
\end{enumerate}

Clearly, in either case, $h_{n+1}$ satisfies Conditions 1-2. Let $h:=\bigcup_{n\ge 0}h_n$. It is easy to verify that $h$ is a bijective function from $\textit{dom}(\Gamma)$ to $\textit{ran}(\Theta)$ such that $\Gamma=\Theta\circ h$. Therefore, $\Gamma$ and $\Theta$ are recursively isomorphic.
\end{proof}

Now we are in the position to present a proof for Theorem~\ref{thm:main1}.

\begin{proof}[Proof of Theorem~\ref{thm:main1}]
Since recursive isomorphism is an equivalence relation, it suffices to prove that any arbitrary universal KRF $\Gamma$ is recursively isomorphic to the canonical KRF $\Theta$. To establish this, we first apply Lemma~\ref{lem:gamma2theta} to demonstrate an injective reduction from $\Gamma$ to $\Theta$. We then use Lemma~\ref{lem:theta2gamma} to show an injective reduction from $\Theta$ to $\Gamma$. By Lemma~\ref{lem:injective2bijective}, these two injective reductions allow us to construct a recursive isomorphism from $\Gamma$ and $\Theta$, thereby completing the proof.
\end{proof}

}

\end{document}